\documentclass[sigconf]{acmart}
\usepackage{wrapfig}
\usepackage{amsmath}
\usepackage{mathtools}
\usepackage{amsthm}
\usepackage{bm}
\usepackage{multirow}

\AtBeginDocument{%
  \providecommand\BibTeX{{%
    \normalfont B\kern-0.5em{\scshape i\kern-0.25em b}\kern-0.8em\TeX}}}

\newcommand{\chen}[1]{\textcolor{black}{#1}}

\usepackage{amsmath,amsfonts,bm,amsthm}

\newcommand{\mset}[1]{\left\{\kern-.5em\left\{ #1 \right\}\kern-.5em\right\}}
\newcommand{\mmset}[1]{\{\kern-.4em\{ #1 \}\kern-.4em\}}

\makeatletter
\newtheorem*{rep@theorem}{\rep@title}
\newcommand{\newreptheorem}[2]{%
\newenvironment{rep#1}[1]{%
 \def\rep@title{#2 \ref{##1}}%
 \begin{rep@theorem}}%
 {\end{rep@theorem}}}
\makeatother

\makeatletter
\newcommand{\subalign}[1]{%
  \vcenter{%
    \Let@ \restore@math@cr \default@tag
    \baselineskip\fontdimen10 \scriptfont\tw@
    \advance\baselineskip\fontdimen12 \scriptfont\tw@
    \lineskip\thr@@\fontdimen8 \scriptfont\thr@@
    \lineskiplimit\lineskip
    \ialign{\hfil$\m@th\scriptstyle##$&$\m@th\scriptstyle{}##$\crcr
      #1\crcr
    }%
  }
}
\makeatletter

\def\eqref#1{equation~\ref{#1}}

\def\1{\bm{1}}

\DeclareMathAlphabet{\mathsfit}{\encodingdefault}{\sfdefault}{m}{sl}
\SetMathAlphabet{\mathsfit}{bold}{\encodingdefault}{\sfdefault}{bx}{n}

\begin{document}

\title{Unifying Invariance and Spuriousity for Graph Out-of-Distribution via Probability of Necessity and Sufficiency}

\author{Xuexin Chen}
\affiliation{%
  \institution{Guangdong University of Technology}
  \streetaddress{P.O. Box 1212}
  \city{Guangzhou}
  \state{Guangdong}
  \country{China}
  \postcode{43017-6221}
}

\author{Ruichu Cai*}
\affiliation{%
  \institution{Guangdong University of Technology}
  \streetaddress{P.O. Box 1212}
  \city{Guangzhou}
  \state{Guangdong}
  \country{China}
  \postcode{43017-6221}
}
\email{cairuichu@gmail.com}

\author{Kaitao Zheng}
\affiliation{%
  \institution{Guangdong University of Technology}
  \streetaddress{P.O. Box 1212}
  \city{Guangzhou}
  \state{Guangdong}
  \country{China}
  \postcode{43017-6221}
}


\author{Zhifan Jiang}
\affiliation{%
  \institution{Guangdong University of Technology}
  \streetaddress{P.O. Box 1212}
  \city{Guangzhou}
  \state{Guangdong}
  \country{China}
  \postcode{43017-6221}
}

\author{Zhengting Huang}
\affiliation{%
  \institution{Guangdong University of Technology}
  \streetaddress{P.O. Box 1212}
  \city{Guangzhou}
  \state{Guangdong}
  \country{China}
  \postcode{43017-6221}
}

\author{Zhifeng Hao}
\affiliation{%
  \institution{Shantou University}
  \streetaddress{8600 Datapoint Drive}
  \city{Shantou}
  \state{Guangdong}
  \country{China}
  \postcode{78229}}


\author{Zijian Li}
\affiliation{%
  \institution{Mohamed bin Zayed University of Artificial Intelligence}
  \streetaddress{1 Th{\o}rv{\"a}ld Circle}
  \city{Masdar City}
  \country{Abu Dhabi}}


\renewcommand{\shortauthors}{Trovato and Tobin, et al.}
\newtheorem{assumption}[theorem]{Assumption}

\begin{abstract}
Graph Out-of-Distribution (OOD), requiring that models trained on biased data generalize to the unseen test data, has a massive of real-world applications. One of the most mainstream methods is to extract the invariant subgraph by aligning the original and augmented data with the help of environment augmentation. However, these solutions might lead to the loss or redundancy of semantic subgraph and further result in suboptimal generalization. To address this challenge, we propose a unified framework to exploit the \textbf{P}robability of \textbf{N}ecessity and \textbf{S}ufficiency to extract the \textbf{I}nvariant \textbf{S}ubstructure (\textbf{PNSIS}). Beyond that, this framework further leverages the spurious subgraph to boost the generalization performance in an ensemble manner to enhance the robustness on the noise data. Specificially, we first consider the data generation process for graph data. Under mild conditions, we show that the invariant subgraph can be extracted by minimizing an upper bound, which is built on the theoretical advance of probability of necessity and sufficiency. To further bridge the theory and algorithm, we devise the \textbf{PNSIS} model, which involves an invariant subgraph extractor for invariant graph learning as well invariant and spurious subgraph classifiers for generalization enhancement. Experimental results demonstrate that our \textbf{PNSIS} model outperforms the state-of-the-art techniques on graph OOD on several benchmarks, highlighting the effectiveness in real-world scenarios.

\end{abstract}




\maketitle

\section{Introduction}
\begin{figure*}[t!]
	\centering
	\includegraphics[width=2.1\columnwidth]{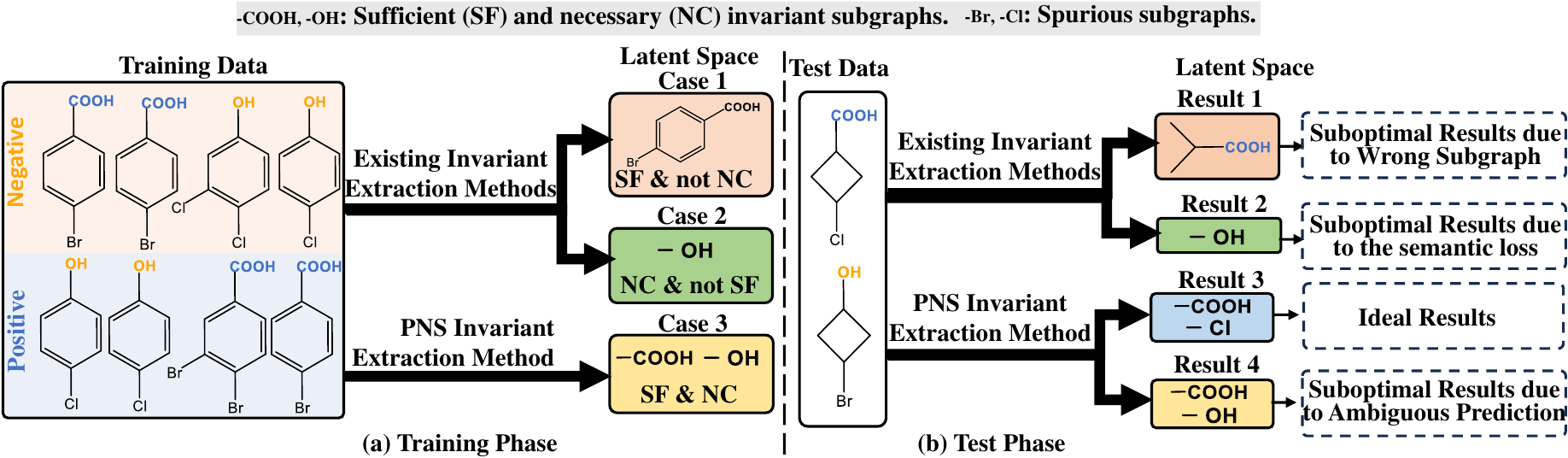} 
  \vspace*{-2.5mm}
	\caption{Illustration of graph OOD methods with invariant subgraph learning, 
 (a) In the training phase, the existing invariant extraction methods might lead to sufficient but not necessary subgraphs (pink block) and the necessary but not sufficient subgraphs (green block) according to two extreme optimization goals. The PNS invariant extraction method can extract the correct invariant subgraph. (b) In the test phase, the conventional invariant methods, that extract sufficient but not necessary latent subgraphs, might generate wrong subgraphs and further lead to suboptimal results (pink block). The methods that extract the necessary but not sufficient subgraphs might lead to the loss of semantic information (green block). When the noise-label data exists, the model might generate ambiguous predictions even if the correct subgraphs have been extracted. (yellow block). Ideal performance can be achieved by combining necessary and sufficient invariant subgraphs as well as spurious subgraphs.}
 \label{fig:motivation}
 \vspace*{-1mm}
\end{figure*}







Graph representation learning with \textbf{G}raph \textbf{N}eural \textbf{N}etworks (GNNs) have gained remarkable success in complicated problems such as 
social recommendation, intelligent transportation, etc. 
Despite their enormous success, the existing GNNs generally assume that the testing and training graph data are \textbf{i}ndependently sampled from the \textbf{i}dentical \textbf{d}istribution (I.I.D.). However, the validity of this assumption is often difficult to guarantee in real-world scenarios. 

    To solve the Out Of Distribution (OOD) challenge of graph data, one of the most popular methods \cite{li2022ood,li2022graphde,zhao2020uncertainty,liu2023flood,sui2022causal,sui2022causal} is to extract domain-invariant features for graph data. Previous, Li et al. \cite{li2022ood} address
the OOD challenge by eliminating the statistical dependence between relevant and irrelevant graph representations; Since the spurious correlations lead to the poor generalization of GNNs, Fan et.al \cite{fan2023generalizing} leverage the stable learning to extract the invariant components. Recently, several researchers have considered the environment-augmentation to extract invariant representation. 
Liu et.al \cite{liu2022graph} perform rationale-environment separation to address the graph-ood challenge; Chen et.al \cite{chen2023does} further use environment augmentation to boost the extraction of invariant features of graph data; And Li \cite{li2023identifying} employ data augmentation techniques to provide identification guarantees for the invariant latent variables. In summary, these methods aim to achieve the invariant representation by balancing two goals 1) aligning the original and augmented feature space and 2) minimizing the prediction error on training data. 







Although existing methods with environmental augmentation have achieved outstanding performance in graph OOD, they can hardly extract ideal invariant subgraphs due to the difficulty of the trade-off between invariant alignment and prediction accuracy. 
To better understand this phenomenon, we provide a toy example of molecular property classification, where the negative and positive labels are decided by special functions like $-\text{COOH}$ and $-\text{OH}$, respectively. Existing methods that balance the feature alignment restriction and the classification loss might result in two extreme cases. As shown in Case 1 of Figure \ref{fig:motivation}(a), when the models put more weight on the optimization of the classification loss, the \textit{sufficient but not necessary} latent subgraphs are extracted, i.e., the subgraph ``Benzene Ring'' is involved in the subgraph for classification in the training phase. However, the wrong invariant subgraphs are extracted in the test phase as shown in the pink block of Figure \ref{fig:motivation}(b). Since the invariant subgraphs in the test phase are different from those in the training phase, the model can hardly achieve optimal performance. As shown in Case 2 of Figure \ref{fig:motivation}(a), 
\chen{when there is an overlap between invariant subgraphs of different categories,}
the model might achieve the \textit{necessary but not sufficient} latent subgraphs, 
e.g.,
the function group ``$-$OH'', 
\chen{is shared by ``$-$COOH'' and ``$-$OH''. In this case, the model may fail to distinguish between samples containing ``$-$OH'' and those containing ``$-$COOH'',}
leading to ambiguous prediction. 
To solve this problem, the spurious subgraphs, i.e., -Br and -Cl, which are related to the semantic-relevant subgraphs, should be taken into consideration as discussed. 

Based on the examples above, an intuitive solution to the graph OOD problem is to 1) extract the \textit{sufficient and necessary} latent subgraphs and 2) employ the invariant and spurious subgraphs for prediction, which is shown in the blue block of Figure \ref{fig:motivation}(b). Under this intuition, we propose a learning framework to exploit the \textbf{P}robability of \textbf{N}ecessity and 
\textbf{S}ufficiency to extract the \textbf{I}nvariant \textbf{S}ubstructure (PNSIS). Technologically, we first employ the theory of probability of necessity and sufficiency in causality and devise a PNS-invariant subgraph extractor to extract the necessity and sufficiency invariant subgraphs for Graph OOD. Specifically, the PNS-invariant subgraph contains a sufficient subgraph extractor and a necessary subgraph extractor, where the PNS-invariant subgraphs can be extracted by optimizing the PNS upper bound. To further leverage the spurious subgraphs, the proposed PNSIS employs an ensemble train strategy with the spurious subgraphs classifier to introduce the spurious information in the test period. The proposed PNSIS is validated on several mainstream simulated and real-world benchmarks for application evaluation. The impressive performance that outperforms state-of-the-art methods demonstrates the effectiveness of our method.
\section{Related Work}
\subsection{Graph Out-of-Distrubtion. }
In this subsection, we provide an introduction to domain generalization of graph classification \cite{fan2023generalizing,yang2022learning,guo2020graseq,liu2022graph,chen2022learning,10027780}. Existing works on out-of-distribution (OOD) \cite{shen2021towards} mainly focus on the fields of computer vision \cite{zhang2021deep,zhang2022multi} and natural language processing \cite{chen2021hiddencut}, but the OOD challenge on graph-structured data receives less attention. Considering that the existing GNNs lack out-of-distribution generalization \cite{li2022graphde,zhao2020uncertainty,liu2023flood,sui2022causal,sun2022does,li2022ood,li2023identifying}, Li et. al \cite{li2021ood} proposed OOD-GNN to tackle the graph OOD (OOD) challenge by addressing the statistical dependence between relevant and irrelevant graph representations. Recognizing that spurious correlations often undermine the generalization of graph neural networks (GNN), Fan et.al propose the StableGNN \cite{fan2023generalizing}, which extracts causal representation for GNNs with the help of stable learning. Aiming to mitigate the selection bias behind graph-structured data, Wu et. al further proposes the DIR model \cite{wu2022discovering} to mine the invariant causal rationales via causal intervention. These methods essentially employ causal effect estimation to make invariant and spurious subgraphs independent. And the augmentation-based model is another type of important method. Liu et.al \cite{10.1145/3534678.3539347} employ augmentation to improve the robustness and decompose the observed graph into the environment part and the rationale part. Recently, Chen et.al \cite{chen2022learning, chen2023does} investigate the usefulness of
the augmented environment information from the theoretical perspective. And Li et. al \cite{li2023identifying} further consider a concrete scenario of graph OOD, i.e., molecular property prediction from the perspective of latent variables identification \cite{li2023subspace}. Although the aforementioned methods mitigate the distribution shift of graph data to some extent, they can not extract the invariant subgraphs with Necessity and Sufficiency \cite{yang2023invariant}. Moreover, as \cite{eastwood2023spuriosity} discussed, the spurious subgraphs also play a critical role when the data with noisy label \cite{liu2015classification,wu2024time,bai2023subclass}. In this paper, we propose the PNSIS framework, which unifies the extraction of the invariant latent subgraph with probability of necessity and sufficiency and the exploitation of spurious subgraphs via an ensemble manner. 

\subsection{Probability of Necessity and Sufficiency}

As the probability of causation, the Probability of Necessity and Sufficiency (PNS) can be used to measure the ``if and only if'' of the relationship between two events. Additionally,  the Probability of Necessity (PN) and Probability of  Sufficiency (PS) are used to evaluate the ``sufficiency cause'' and ``necessity cause'', respectively. 
Pearl \cite{pearl2000models} and Tian and Pearl \cite{tian2000probabilities} formulated precise meanings for the probabilities of causation using structural causal models. 
The issue of the identifiability of PNS initially attracted widespread attention 
\cite{galles1998axiomatic,halpern2000axiomatizing,pearl2009causality,cai2023learning,tian2000probabilities, li2019unit,li2022unit, mueller2021causes, dawid2017probability, li2022probabilities, gleiss2019quantifying, zhang2022partial, li2022bounds}.
Kuroki and Cai [24] and Tian and Pearl [56] demonstrated how to bound these quantities from data obtained in experimental and observational studies to solve this problem.
These bounds lie within the range  which the probability of causation must lie, however, it has been pointed out that these bounds are too wide to assess the probability of causation.
To overcome this difficulty, 
Pearl demonstrated that identifying the probabilities of causation requires specific functional relationships between the causes and its outcomes \cite{pearl2000models}. 
Recently, incorporating PNS into various application scenarios has also attracted much attention and currently has many applications \cite{tan2022learning, galhotra2021explaining,watson2021local,cai2022probability,mueller2021causes,beckers2021causal,shingaki2021identification}. For example, in ML explainability, CF$^2$ \cite{tan2022learning}, LEWIS \cite{galhotra2021explaining}, LENS \cite{watson2021local} and NSEG \cite{cai2022probability} use sufficiency or necessity to measure the contribution of input feature subsets to model's predictions. In the causal effect estimation problem \cite{mueller2021causes,beckers2021causal,shingaki2021identification}, it can be used to learn individual responses from population data \cite{mueller2021causes}. 
In the out-of-distribution generalization problem, CaSN employs PNS to extract domain-invariant information \cite{yang2023invariant}. 
Although CaSN is effective in extracting sufficient and necessary invariant representations, it does not capture spurious features that could improve generalization prediction. Therefore, its generalization prediction performance may be suboptimal. Furthermore, the invariant representation learned by CaSN lacks interpretability, while our PNSIS provides interpretable invariant subgraphs.



\section{Notations and Problem Formulation}
Given the training graphs $\mathcal{G}_{train} = \{(G_1, Y_1), ...,$ $(G_N, Y_N)\}$, where $G_n= (\mathcal{V}_n, \mathcal{E}_n, \mathbf{A}_n)$ represents the $n$-th graph data in the training set and $Y_n$ is the corresponding label, with the set of nodes $\mathcal{V}_n$, the set of edges $\mathcal{E}_n$ and the associated adjacency matrix $\mathbf{A}_n \in \mathbb R^{|\mathcal{V}_n| \times |\mathcal{V}_n|}$ where the element of $i$-th row and $j$-th column is 0 if the node pair ($i, j$) has no edges, otherwise it means the edge weight in pair ($i, j$). The node feature matrix of $G_n$ is represented as $\mathbf{X}_n \in \mathbb R^{|\mathcal{V}_n| \times D}$, where $i$-th row means the feature of node $i$. We use $\mathbf{B}_{i,:}$ to represent all of the $i$-th row's elements of a matrix $\mathbf{B}$.

In this paper, we investigate the problem of OOD generalization problem on graph data. We begin with the data generation process as shown in Figure \ref{fig:causal_graph22}, which is known as Partially Informative Invariant Features (PIIF) \cite{arjovsky2019invariant}. Let $G$ denote an observed graph variable, which is pointed to by $G^c$ and $G^s$, which means that $G^c$ and $G^s$ control the generation of the observed graphs. 
Let $E$ denote the environment indices and $E \to G^s$ means that the spurious subgraph in each environment may be different. Remarkably, we find the label $Y$ is the mediator between $G^c$ and $G^s$, which means that $G^c$ and $G^s$ is independent given $Y$.

Based on the aforementioned generation process, the task is to learn a GNN model $h_{\varphi}(\cdot)$ with parameter $\varphi$ to accurately predict the label of the testing graphs $\mathcal{G}_{test}$, where the distribution $\Psi(\mathcal{G}_{train}) \ne \Psi(\mathcal{G}_{test})$. 
Note that the test distribution is unknown in the OOD setting. Moreover, we assume the existence of $M$ environments $\{\mathcal{T}_1, \mathcal{T}_2,$ $..., \mathcal{T}_M\}$ in training graphs $\mathcal{G}_{train}$. For each environment $\mathcal{T}_i$, data are drawn from a distinct distribution, represented as $p_{\mathbf{A}, \mathbf{X}, Y|E_{\mathcal{T}_i}}$, where the variables $\mathbf{A}$, $\mathbf{X}$, $Y$ correspond to adjacency matrices, node feature matrices and labels, respectively.

\begin{figure}[t]
\centering
\includegraphics[width=0.15\textwidth]{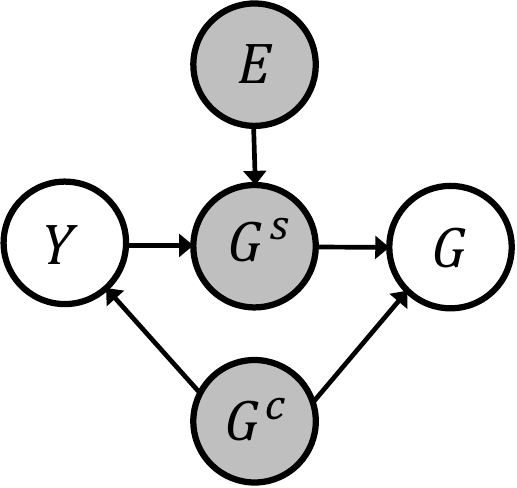}
\caption{PIIF SCM\cite{arjovsky2019invariant}. Inside this graph, the noises are omitted for brevity, where the grey and white nodes denote the latent and observed variables, respectively.}
\vspace{-3mm}
\label{fig:causal_graph22}
\end{figure}

\section{Invariant Subgraph Learning via PNS Upper Bound Optimization}
This section describes a theory for extracting sufficient and necessary invariant subgraphs based on PNS. 
Specifically, we first introduce the basic model used to build our invariant learning theory. This basic model is a graph generalization model that is jointly predicted by two different invariant subgraph extractors which we call sufficient subgraph extractor and necessary subgraph extractor. 
Secondly, we reformulate the invariant subgraph learning as a trade-off between the sufficient subgraph extractor and the necessary subgraph extractor. 
Thirdly, we designed a PNS upper bound that defined on any two different environments for optimization, to enable these two subgraph  extractors to produce the necessary and sufficient invariant subgraph. 


\subsection{Reformulation of Invariant Subgraph Learning via Sufficiency and Necessity}
We begin by introducing the basic model that is jointly predicted by two different invariant subgraph extractors and its procedure is broken down into the following three phases. First, feed the same input graph data ($G_n$, $Y_n$) into two invariant subgraph extractors $f(\cdot; \Theta^{sf})$ and $f(\cdot; \Theta^{nc})$ with parameters given by $\Theta^{sf}$ and $\Theta^{nc}$.  $f(\cdot; \Phi^{sf})$ is designed to extract sufficient invariant subgraphs and $f(\cdot; \Phi^{nc})$ is to extract necessary invariant subgraphs, which are called sufficient subgraph extractor and necessary subgraph extractor respectively in this paper (specific details on how to design these two extractors are given in the following sections). Second, feed the output graph from $f(\cdot; \Theta^{sf})$ and $f(\cdot; \Theta^{nc})$ into two different GNNs $g(\cdot;\Phi^{sf})$ and $g(\cdot;\Phi^{nc})$ for classification with parameters given by $\Phi^{sf}$ and $\Phi^{nc}$. Third, average the outputs of the two classifiers and serve as the final prediction. In summary, we formalize the invariant subgraph learning model as follows.
\begin{equation}\label{equ:model}
    \hat{Y}_n = 0.5 \cdot g(f(\mathbf{A}_n, \mathbf{X}_n; \Theta^{sf}); \Phi^{sf}) + 0.5 \cdot g(f(\mathbf{A}_n, \mathbf{X}_n; \Theta^{nc}), \Phi^{nc}), 
\end{equation}



When the outputs of these two subgraph extractors tend to be the same, i.e., their outputs are both sufficient and necessary invariant subgraphs, the prediction accuracy of the model is high. The reasons are as follows. If the sufficient subgraph extractor $f(\cdot;\Theta^{sf})$ captures sufficient but unnecessary invariant subgraph, the prediction may be incorrect since the testing graph may not contain this subgraph. 
If the necessary subgraph extractor $f(\cdot;\Theta^{nc})$ captures the necessary but insufficient invariant subgraph, the prediction may be also incorrect since other classes of the testing graph may also contain this subgraph. 
In other words, when the outputs of these two extractors are highly inconsistent, the model's prediction accuracy will be low. 
Therefore, the invariant subgraph learning problem is transformed into a trade-off between these two subgraph extractors. 
In the next section, we describe how to design and optimize this trade-off.


\subsection{PNS Upper Bound for Invariant Subgraph Extraction}
Probability of Necessity and Sufficiency (PNS) \cite{pearl2022probabilities} is to describe the probability that event $A$ occurs if and only if event $B$ occurs. This probability
operates in two events to compute the probability that event $A$ is necessary and sufficient cause for event $B$. 
In this section, we describe a PNS upper bound which is defined on any two different environments for optimization to ensure that the output of the sufficient or necessary subgraph extractor is the necessary and sufficient cause (subgraph format) of the label of the given input graph. To achieve this goal, we first define the PNS risk in a single environment, and then extend its definition to any two environments through upper bound derivation. 
\subsubsection{PNS Risk.} 
Let event $A$ denote $G^c = f(\mathbf{A}_n, \mathbf{X}_n; \Theta^{sf})$ and event $B$ denote  $g(f(\mathbf{A}_n, \mathbf{X}_n; \Theta^{sf}); \Phi^{sf}) = y_n$. 
Therefore the occurrence of events $A$ and $B$ means that the prediction of an invariant subgraph generated from $f(\mathbf{A}_n, \mathbf{X}_n;\Phi^{sf})$ is the same as the label $y_n$.
Further, we let $\bar A$ denote $G^c = f(\mathbf{A}_n, \mathbf{X}_n; \Theta^{nc})$ and event $\bar B$ denote  $g(f(\mathbf{A}_n, \mathbf{X}_n; \Theta^{nc}); \Phi^{nc}) \ne y_n$. Thus the occurrence of events $\bar A$ and $\bar B$ means that the prediction of an invariant subgraph generated from $f(\mathbf{A}_n, \mathbf{X}_n;\Phi^{nc})$ is different from the label $y_n$. Moreover, the definition of PNS for these events is as follows \cite{pearl2022probabilities}.
\begin{definition} (Probability of necessity, PN)\label{equ:ddpn}
\begin{equation}
    PN = P(\bar{B}_{\bar A}|A, B),
\end{equation}
PN is the probability that, given that events $A$ and $B$ both occur initially, event $B$ does not occur after event $A$ is changed from occurring to not occurring.
\end{definition}
\begin{definition}\label{equ:ddps}
(Probability of sufficiency, PS)
\begin{equation}
    PS = P({B}_{A}|\bar A, \bar B)
\end{equation}
PS is the probability that, given that events $A$ and $B$ both did not occur initially, event $B$ occurs after event $A$ is changed from not occurring to occurring.
\end{definition}
\begin{definition} (Probability of necessity and sufficiency, PNS)
    \begin{equation}
    \label{equ:pns_obj}
        PNS = PN \cdot P(A, B) + PS \cdot P(\bar A, \bar B)
    \end{equation}
PNS is the sum of PN and PS, each multiplied by the probability of its corresponding condition. PNS measures the probability that event $A$ is a necessary and sufficient cause for event $B$.
\end{definition}

However, the exact computation of PNS is not tractable in most cases since PN and PS require counterfactual reasoning, and the counterfactual dates are usually difficult to obtain in the real world. 
We adopt the suggestion proposed by Yang \cite{yang2023invariant} and calculate PNS risk based on the observed data in the following way. Let $P(G^c |\mathbf{A}_n, \mathbf{X}_n; \Theta_i)$ be a multivariate Bernoulli distribution for invariant subgraph variable $G^c$ parameterized by the output of $f(\mathbf{A}_n, \mathbf{X}_n; \Theta_i)$.
PNS risk which is defined on environment $E_{\mathcal{T}_i}$, is defined as follows.
\begin{definition} (PNS risk)
The following PNS risk is defined over the distribution of an environment $E_{\mathcal{T}_i}$.
\begin{equation}\label{equ:estimated_pns}
\begin{aligned}
    r^{\mathcal{T}_i}_{ns}(\Theta^{sf}, &\Phi^{sf}, \Theta^{nc}, \Phi^{nc}) :=\mathbb E_{(\mathbf{A}_n, \mathbf{X}_n, y_n) \sim p_{\mathbf{A}, \mathbf{X}, Y|E_{\mathcal{T}_i}}}\big[\\
    &\mathbb{E}_{G_j^c \sim P(G^c |\mathbf{A}_n, \mathbf{X}_n; \Theta^{sf})}\mathbb{I}[g(G_j^c; \Phi^{sf}) \ne  y_n]\\
    &+ \mathbb{E}_{G_j^c \sim P(G^c |\mathbf{A}_n, \mathbf{X}_n; \Theta^{nc})}\mathbb{I}[g(G_j^c; \Phi^{nc}) =  y_n]\big],
\end{aligned}
\end{equation}
\end{definition}
PNS risk measures the negative probability that the invariant subgraphs generated by $f(\cdot; \Theta^{sf})$ and $f(\cdot; \Theta^{nc})$ are both the sufficient and necessary cause of $y_n$, respectively. To generate the sufficient and necessary invariant subgraph, we extend $r^{\mathcal{T}_i}_{ns}(\cdot)$ in Eq.~\ref{equ:estimated_pns} from $E_{\mathcal{T}_i}$ to any two environments in the following section.

\subsubsection{PNS Upper Bound.}
In this section, we propose a graph structure distance to derive an upper bound of the PNS risk.
Let $\mathbf{C}_n \in \mathbb R^{|\mathcal{V}| \times |\mathcal{V}| \times (D+2)}$ be the collection of features of $G_n$, where $\mathbf{C}_{n_{:,:, i}}$ ($i \leq D$) denotes the diagonal matrix with diagonal entries being the feature vector of node $i$, $\mathbf{C}_{n_{:,:, D+1}}$ and  $\mathbf{C}_{n_{:,:, D+2}}$ represent the adjacency matrix and the euclidean distance matrix between each node feature respectively. Let PMP denote Power-sum Multi-symmetric Polynomials and we use the following permutation invariant function \cite{maron2019provably} which is as powerful as the $3$-WL graph isomorphism test, to encode the structure information of a graph data into a vector.
\begin{equation}\label{equ:fwl}
    \mathbf{h}_n  =  \sum_{(i_1, i_2) \in [V_n]^2}\mathbf{C}_{n_{i_1, i_2, :}} \big\| \text{PMP}(\mmset{ (\mathbf{C}_{n_{k, i_2, :}}, \mathbf{C}_{n_{i_1, k, :}}) \big\vert k \in [V_n]}),
\end{equation}
where $\mmset$ represents a multiset, $V_n$ is the number of nodes of the graph $G_n$, $[V_n] = \{1,2,..., V_n\}$ and $[V_n]^2$ denote the Cartesian product of $[V_n]$ with itself. Please see Appendix \ref{sec:pmp} for more details about this graph structure representation function. 
After this, let $p_{\mathbf{h}|E_{\mathcal{T}_i}}$$(i=1,..., M)$ denote the distribution of graph structure representation in Eq.~\ref{equ:fwl} on different environment data, respectively. 
\begin{figure*}[t]
	\centering
	\includegraphics[width=2.\columnwidth]{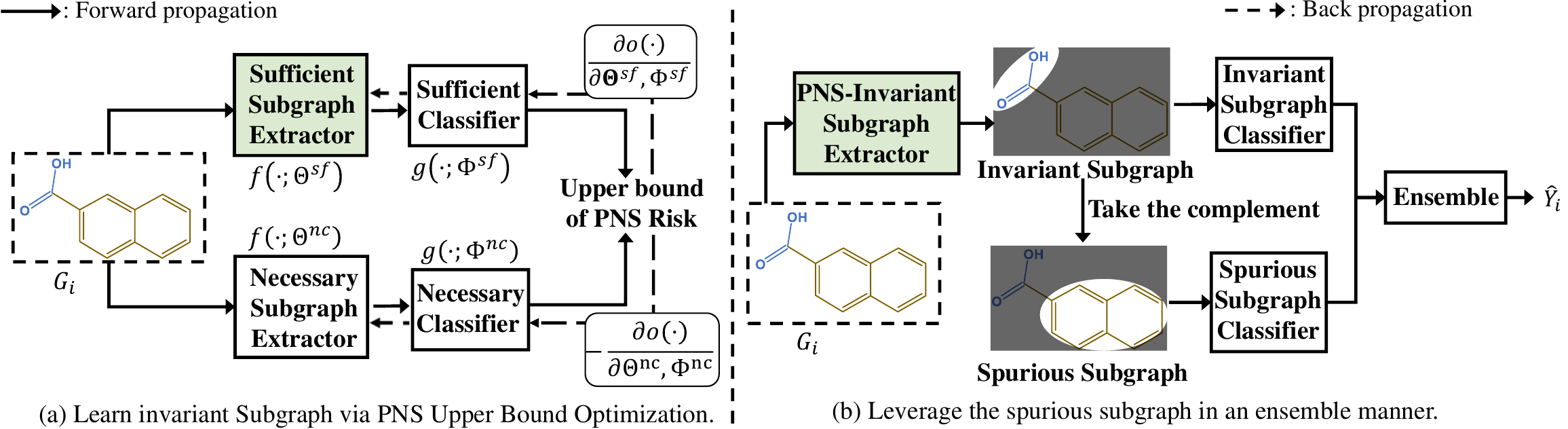} 
	\caption{The illustration of the PNSIS framework.  (a) The left side of the figure denotes the subgraph extractors that are used to extract the sufficient and necessary invariant subgraphs by optimizing the upper bound. (b) The right side of the figure denotes the ensemble inference phase, which includes an invariant subgraph classifier and a spurious subgraph classifier
 }
 \label{fig:model}
\end{figure*}
We define the graph structure distance as follows.
\begin{definition}\label{def:distance} (Graph Structure Distance, GSD).
The structure distance between environments $E_{\mathcal{T}_i}$ and $E_{\mathcal{T}_j}$ can be formalized as follows:
\begin{equation}\label{equ:distance1}
\begin{aligned}
    d_{sd}^{\mathcal{T}_i \leftrightarrow \mathcal{T}_j}(\mathcal{T}_i, &\mathcal{T}_j) =dist(p_{\mathbf h|E_{\mathcal{T}_i}}, p_{\mathbf h|E_{\mathcal{T}_j}})\\
    &+\mathbb{E}_{\mathbf{X}_i \sim p_{\mathbf X|E_{\mathcal{T}_i}}}\big[\mathbb E_{\mathbf{X}_j \sim p_{\mathbf X|E_{\mathcal{T}_j}}} [\|\mathbf{X}_i - \mathbf{X}_j\|_2]\big]
\end{aligned}
\end{equation}
\end{definition}
Note that $dist(\cdot, \cdot)$ in Eq.~\ref{equ:distance1} can be any distance metric for distribution and we employ the total variation distance in this paper. Hence, our GSD is affected by structure and semantic (node features) information. Moreover, our GSD satisfies the three axioms for a general metric.
\begin{theorem}\label{thm:distance}
We make the following assumption:
\begin{itemize}
    \item A1. For any two graphs $G_i$, $G_j$ from different environments, $G_i$ and $G_j$ can always be distinguished by the 3-WL graph isomorphism test.
\end{itemize}
Graph Structure Distance (GSD) satisfies the three axioms for a general metric, to be specific, it satisfies the following conditions:
\begin{enumerate}
    \item[1)] $d_{sd}^{\mathcal{T}_i \leftrightarrow \mathcal{T}_j}$$\geq$$0$ and $d_{sd}^{\mathcal{T}_i \leftrightarrow \mathcal{T}_j}$$=$$0$ if and only if $p_{\mathbf{A},\mathbf{X},Y, |E_{\mathcal{T}_i}}$$=$$ p_{\mathbf{A},\mathbf{X},Y, |E_{\mathcal{T}_j}}$;
    \item[2)] $d_{sd}^{\mathcal{T}_i \leftrightarrow \mathcal{T}_j}\left(\mathcal{T}_i, \mathcal{T}_j\right)=d_{sd}^{\mathcal{T}_j \leftrightarrow \mathcal{T}_i}\left(\mathcal{T}_j, \mathcal{T}_i\right)$ (symmetric);
    \item[3)] $d_{sd}^{\mathcal{T}_i \leftrightarrow \mathcal{T}_j} \leq d_{sd}^{\mathcal{T}_i \leftrightarrow \mathcal{T}_k}
    +d_{S D}^{\mathcal{T}_k \leftrightarrow \mathcal{T}_j}$ (triangle inequality).
\end{enumerate}
\end{theorem}
The detailed proof can be found in Appendix \ref{sec:distance}.
Next, we provide the upper bound for PNS risk in Eq.~\ref{equ:estimated_pns} through Theorem \ref{thm:bound}.
\begin{theorem}\label{thm:bound}
\textbf{(Generalization Bound)}
    We make the following assumption:
    \begin{itemize}
        \item A2: For two distinct environment distributions $p_{\mathbf{A},\mathbf{X},Y, |E_{\mathcal{T}_i}}$ and $p_{\mathbf{A},\mathbf{X},Y, |E_{\mathcal{T}_i}}$,        
        assume a positive value $K$ exists that satisfies the following inequality:
\begin{equation}
\begin{aligned}
&|p_{\mathbf{A},\mathbf{X},Y|E_{\mathcal{T}_i}}-p_{\mathbf{A},\mathbf{X},Y|E_{\mathcal{T}_j}}|\\
 \leq &K\cdot\big(|
 p_{\mathbf{h} |E_{\mathcal{T}_i}}
 -
p_{\mathbf{h} |E_{\mathcal{T}_j}}
 |+\mathbb{E}_{\mathbf{X}_i \sim p_{\mathbf{X}|E_{\mathcal{T}_i}}}\big[\mathbb E_{\mathbf{X}_j \sim p_{\mathbf{X} |E_{\mathcal{T}_j}}} [\|\mathbf{X}_i - \mathbf{X}_j\|_2]\big]\big) \\
= &K \cdot d_{sd}^{\mathcal{T}_i \leftrightarrow \mathcal{T}_j}\left(\mathcal{T}_i, \mathcal{T}_j\right).
\end{aligned}
\end{equation}
    \end{itemize}
Based on the aforementioned definition and assumption, we
propose the generalization bound for PNS risk in Eq.~\ref{equ:estimated_pns}.
\begin{equation}\label{equ:bound}
\begin{aligned}
    &r^{\mathcal{T}_i}_{ns}(\Theta^{sf}, \Phi^{sf}, \Theta^{nc}, \Phi^{nc})
    \leq  r^{\mathcal{T}_j}_{ns}(\Theta^{sf}, \Phi^{sf}, \Theta^{nc}, \Phi^{nc}) \\&+ K\cdot \big(|p_{\mathbf{h}|E_{\mathcal{T}_i}}-p_{\mathbf{h}|E_{\mathcal{T}_j}}|
    +\mathbb{E}_{\mathbf{X}_i \sim p_{\mathbf{X}|E_{\mathcal{T}_i}}}\big[\mathbb E_{\mathbf{X}_j \sim p_{\mathbf{X}|E_{\mathcal{T}_j}}} [\|\mathbf{X}_i - \mathbf{X}_j\|_2]\big]\big) \\
    = &r^{\mathcal{T}_j}_{ns}(\Theta^{sf}, \Phi^{sf}, \Theta^{nc}, \Phi^{nc}) + K \cdot d_{sd}^{\mathcal{T}_i \leftrightarrow \mathcal{T}_j}\left(\mathcal{T}_i, \mathcal{T}_j\right) + \lambda,   
\end{aligned}
\end{equation}
where $K$, $\lambda$ are constants.
\end{theorem}

\chen{The detailed proof can be found in Appendix \ref{sec:bound}.} 
Finally, combining Eqs.~\ref{equ:estimated_pns},~\ref{equ:bound}, the final objective function $o(\cdot)$ is as follows:
\begin{equation}\label{equ:final_obj}
\begin{aligned}
&\min o(\Theta^{sf}, \Theta^{nc}, \Phi^{sf}, \Phi^{nc}) =\mathbb E_{(\mathbf{A}_n, \mathbf{X}_n, y_n) \sim p_{\mathbf{A}, \mathbf{X}, Y|E_{\mathcal{T}_i}}}\big[\\
    &\mathbb{E}_{G_j^c \sim P(G^c |\mathbf{A}_n, \mathbf{X}_n; \Theta^{sf})}\mathbb{I}[g(G_j^c; \Phi^{sf}) \ne  y_n]\\
    &+ \mathbb{E}_{G_j^c \sim P(G^c |\mathbf{A}_n, \mathbf{X}_n; \Theta^{nc})}\mathbb{I}[g(G_j^c; \Phi^{nc}) =  y_n]\big]\\
    &+ K\cdot \big(|p_{\mathbf{h}|E_{\mathcal{T}_i}}-p_{\mathbf{h}|E_{\mathcal{T}_j}}|
    +\mathbb{E}_{\mathbf{X}_i \sim p_{\mathbf{X}|E_{\mathcal{T}_i}}}\big[\mathbb E_{\mathbf{X}_j \sim p_{\mathbf{X}|E_{\mathcal{T}_j}}} [\|\mathbf{X}_i - \mathbf{X}_j\|_2]\big]\big).
\end{aligned}
\end{equation}
When the objective function in Eq.~\ref{equ:final_obj} gradually converges,
the output subgraphs of $f(\cdot; \Theta^{sf})$ and $f(\cdot; \Theta^{nc})$ are both necessary and sufficient 
invariant subgraphs. In practice, we take $f(\cdot; \Theta^{sf})$ and its classifier $g(\cdot;\Phi^{sf})$ as PNS-invariant subgraph extractor and invariant subgraph classifier, respectively.
%
The next section provides a detailed implementation of Eq.~\ref{equ:final_obj} and further incorporates spurious features for prediction.


\section{Invariant Subgraph Learning Model Incorporating Spurious Subgraph}
Building on the theory of invariant subgraph learning via PNS, we exploit the Probability of Necessity and Sufficiency to extract the Invariant Substructure (PNSIS). Beyond that, we further leverage the spurious subgraph to boost the generalization performance in an ensemble manner. The overall framework of PNSIS is shown in Figure \ref{fig:model}. Specifically, the procedure of our PNSIS is broken down into the following two phases. 
In the first phase, PNSIS adopts a stochastic gradient descent algorithm to train sufficient and necessary subgraph extractors to obtain sufficient and necessary invariant subgraphs, as shown in Figure \ref{fig:model} (a).
In the second phase, using the trained sufficient subgraph extractor and its classifier as a PNS-invariant subgraph extractor and invariant subgraph classifier, respectively, PNSIS trains a spurious subgraph classifier and integrates the outputs of both classifiers into the final prediction, as shown in Figure \ref{fig:model} (b). 
More details on the above two phases are given in the following sections.
\subsection{Model Implementation and Optimization}

In this section, we first provide a concrete implementation of the underlying model in Eq.~\ref{equ:model} used to learn sufficient and necessary invariant subgraphs. Second, using Monte Carlo methods to estimate the loss of the model w.r.t the objective function in Eq.~\ref{equ:final_obj}, and optimize the parameters of the invariant subgraph extractors and classifiers by backpropagation.
\subsubsection{Model Implementation.} Technically, we implement the 
invariant subgraph extractors $f(\mathbf{A}_n, \mathbf{X}_n; \Phi^{sf})$ and $f(\mathbf{A}_n, \mathbf{X}_n; \Phi^{nc})$, and the subgraph classifiers $g(\cdot; \Phi^{sf})$ and $g(\cdot; \Phi^{nc})$
in the following three steps. 
First, use two graph convolution networks (GCNs) $\text{GCN}(\mathbf{G};\Theta^{sf})$, $\text{GCN}(\mathbf{G};\Theta^{nc})$  to generate node embeddings $\mathbf{Z}$ and $\mathbf{Z}'$, respectively. 
Second, 
taking the inner product of $\mathbf{Z}$ and $\mathbf{Z}'$ as the output of $f(\cdot;\Theta^{sf})$ and $f(\cdot;\Theta^{nc})$, respectively,  where each entry $(i, j)$ in the inner product represents the probability of the edge existence for the node pair $(i, j)$. 
Third, use another two GCNs $\text{GCN}(\cdot;\Phi^{sf})$ and $\text{GCN}(\cdot;\Phi^{sf})$ as classifiers, respectively. To summarize, the three steps can be formalized as follows.
\begin{equation}\label{equ:process}
\begin{small}
\begin{aligned}
    &f(\mathbf{A}_n, \mathbf{X}_n;\Theta^{sf}) := \sigma(\mathbf{Z}\mathbf{Z}^\top), \quad\mathbf{Z} = \text{GCN}(\mathbf{A}_n, \mathbf{X}_n; \Theta^{sf}),\\
    &f(\mathbf{A}_n, \mathbf{X}_n; \Theta^{nc}) := \sigma(\mathbf{Z}'\mathbf{Z}'^\top),\quad\mathbf{Z}' = \text{GCN}(\mathbf{A}_n, \mathbf{X}_n; \Theta^{nc}),\\
    &g(f(\cdot; \Theta^{sf}); \Phi^{sf}) := \sigma(\text{GCN}(f(\cdot;\Theta^{sf}); \Phi^{sf})),\\
    &g(f(\cdot; \Theta^{nc}); \Phi^{nc}) := \sigma(\text{GCN}(f(\cdot;\Theta^{nc}); \Phi^{nc})),
\end{aligned}
\end{small}
\end{equation}
where $\sigma(\cdot)$ is the Sigmoid function.
\subsubsection{PNS Upper Bound Estimation.} 
The objective function in Eq.~\ref{equ:final_obj} is estimated as follows. 
Given a graph dataset $\mathcal{G}_{train}$ and randomly drawn two subsets $\mathcal{G}'_{train}$ and $\mathcal{G}''_{train}$ from it, which can be regarded as an approximation of a certain two environments $\mathcal{T}_i$ and $\mathcal{T}_j$ in the training set. First, the last two terms of Eq.~\ref{equ:final_obj} can be estimated by $\mathcal{G}'_{train}$ and $\mathcal{G}''_{train}$. Specifically,
\begin{equation}\label{equ:est_p}
    |p_{\mathbf{h}|E_{\mathcal{T}_j}} - p_{\mathbf{h}|E_{\mathcal{T}_j}}| \approx  C \cdot \sum^{|\mathcal{G}'_{train}|}_{i=1} \sum^{|\mathcal{G}''_{train}|}_{j=1} \|\mathbf{h}_i - \mathbf{h}_j\|_1,
\end{equation}
\begin{equation}\label{equ:est_e}
    \mathbb{E}_{\mathbf{X}_i \sim p_{\mathbf{X}|E_{\mathcal{T}_i}}}\big[\mathbb E_{\mathbf{X}_j \sim p_{\mathbf{X}|E_{\mathcal{T}_j}}} [\|\mathbf{X}_i - \mathbf{X}_j\|_2]\big] \approx C \cdot \sum^{|\mathcal{G}'_{train}|}_{i=1} \sum^{|\mathcal{G}''_{train}|}_{j=1} \|\mathbf{X}_i - \mathbf{X}_j\|_2,
\end{equation}
where $C = 1 / (|\mathcal{G}'_{train}|\cdot|\mathcal{G}''_{train}|)$, $\mathbf{X}_i$ and $\mathbf{X}_j$ are node feature matrices associated with $G_i \in \mathcal{G}'_{train}$ and $G_j \in \mathcal{G}''_{train}$ respectively.
Secondly, to sample $G^c$ from $P(G^c |\mathbf{A}_n, \mathbf{X}_n; \Theta^{sf})$ and $P(G^c |\mathbf{A}_n, \mathbf{X}_n; \Theta^{sf})$, we employ Gumbel-Softmax \cite{jang2016categorical} and denote the sampling results as $\mathcal{G}_{sf}$ and $\mathcal{G}_{nc}$, respectively. The estimation of Eq.~\ref{equ:final_obj} is as follows.
\begin{equation}\label{equ:s}
\begin{aligned}
    \frac{1}{|\mathcal{G}_{train}|} \sum^{|\mathcal{G}_{train}|}_{i=1} &\bigg( \frac{1}{|\mathcal{G}_{sf}|} \sum^{|\mathcal{G}_{sf}|}_{j=1} \mathbb{I}[g(G_j^c; \Phi^{sf}) \ne  y_i]\\
     &+\frac{1}{|\mathcal{G}_{nc}|} \sum^{|\mathcal{G}_{nc}|}_{k=1} \mathbb{I}[g(G_k^c; \Phi^{nc}) =  y_i]\bigg),
\end{aligned}
\end{equation}
where $\mathbb{I}[g(G_j^c; \Phi^{sf}) $$\ne$$  y_i]$, $\mathbb{I}[g(G_j^c; \Phi^{sf}) $$\ne$$  y_i]$ can be replaced by other differentiable losses, e.g., cross-entropy loss. Combining Eqs.~\ref{equ:est_p}-\ref{equ:s}, the estimation of PNS upper bound in Eq.~\ref{equ:final_obj} is obtained. 
Finally, PNSIS adopts the SGD algorithm to optimize the estimated PNS upper bound w.r.t. $\Phi^{sf}$, $\Phi^{nc}$, $\Theta^{sf}$ and $\Theta^{nc}$. Due to the opposite training objectives of necessity and sufficiency, this optimization process can be seen as invariant subgraph adversarial learning.



\subsection{Fuse Invariant and Spuriousity
for Generalization}\label{sec:sp}
To enhance the model generalization, we suggest not only making predictions based on their invariant subgraphs but also taking into account their spurious subgraph in the test set. We accomplish this goal in two steps. 1) Build a model to classify testing graphs based on their spurious subgraphs. 2) Ensemble predictions with trained invariant subgraph classifiers. 
The relevant theories of these steps have been proposed by Eastwood~\cite{eastwood2023spuriosity}. 

\subsubsection{Spurious Subgraph Classifier Optimization}
Create a pseudo-labelled dataset $\mathcal{G}_{pl} = \{(G^c_n, G^s_n, \hat{Y}_n)\}^N_{n=1}$, where we define a spurious subgraph $G_n^s$ as the complement of invariant subgraph $G_n^c$ and $\hat{Y}_n$ is the prediction of invariant classifier given $G^c_n$. Train a graph network classifier $g(\mathbf{A}_n, \mathbf{X}_n; \Phi^{sp})$ on this dataset, which we refer to as the spurious subgraph classifier. 
\subsubsection{Ensemble Prediction}
Given  $\mathcal{G}_{pl}=$$\{(G^c_n, G^s_n,\hat{Y}_n)\}^N_{n=1}$,  trained spurious classifier $g(\cdot;\Phi^{sp})$ and invariant classifier $g(\cdot;\Phi^{sf})$, to simplify the expression, we assume that the OOD task is binary graph classification,  and let  $\epsilon_0 = \sum_{\hat{Y}_n \in \mathcal{G}_{pl}} \hat{Y}_n$, $\epsilon_1 = \frac{1}{N-\epsilon_0} \sum_{n=1}^N(1-\hat{Y}_n)(1-g(G_n^c;\Phi^{sp}))$, $\epsilon_2 = \frac{1}{\epsilon_0} \sum_{n=1}^N \hat{Y}_n g(G^{c}_n; \Phi^{sf})$, $\epsilon_3 = \operatorname{logit}(p(Y=1))$. The ensemble results of the spurious subgraph classifier $g(\cdot; \Phi^{sp})$ and the invariant subgraph classifier $g(\cdot; \Phi^{sf})$ are computed as follows.
\begin{equation}\label{equ:debias}    \hat{Y}=\sigma(\operatorname{logit}(g(G_n^c;\Phi^{sf}))) + \operatorname{logit}\Big(\frac{g(G^s_n;\Phi^{sp}) + \epsilon_0-1}{\epsilon_0 + \epsilon_1-1}\Big) - \epsilon_3.
\end{equation}
%
Starting from the leftmost side of Eq.~\ref{equ:debias}, the first three terms can be viewed as $\log p(Y | G^c, G^s)$, $\log p(Y|G^c)$ and $\log p (Y | G^s)$. 
Thus the key idea of ensemble prediction \cite{eastwood2023spuriosity} is to decompose $P(Y|G^c, G^s)$ into separately estimatable terms.
Note that this decomposition holds true if, given $Y$, $G^c$ and $G^s$ are independent of each other, and the PIIF causal diagram we used in Figure \ref{fig:causal_graph22} satisfies this condition. See Appendix \ref{sec:ensemble} for details.
\section{Experiments}
We evaluate the effectiveness of the proposed PNSIS model on both synthetic and real-world datasets by answering the following questions. \textbf{Q1:} Can the proposed PNSIS model outperforms current state-of-the-art methods under invariant subgraph extraction? \textbf{Q2:} Can the proposed PNS invariant extractor with PNS upper bound restriction well learn the invariant latent subgraphs? \textbf{Q3:} Can the ensemble strategy with spurious subgraphs, which leverages the spurious subgraphs, benefit the model performance?

\begin{table*}[ht]
\centering
\caption{OOD generalization performance on the structure and mixed shifts for synthetic graphs.}
\renewcommand{\arraystretch}{0.50}
\label{tab:ogb_cls1}
\resizebox{\textwidth}{!}{
\begin{tabular}{c|ccc|ccc}
\toprule
    & \multicolumn{3}{c|}{SPMOTIF-STRUC}   
    & \multicolumn{3}{c}{SPMOTIF-MIXED}        \\ \midrule
\textbf{Model}    & BIAS=0.2       & BIAS=0.5  & BIAS=0.8     & BIAS=0.2       & BIAS=0.5           & BIAS=0.8\\   
    \midrule
    \textbf{ERM}        & 0.5807(0.0225) & 0.5998(0.0187) & 0.5692(0.0211) & 0.5808(0.0198) & 0.5725(0.0234) & 0.5252(0.0163) \\
    \textbf{IRM}        & 0.5683(0.0196) & 0.5847(0.0203) & 0.5327(0.0165) & 0.5950(0.0212)  & 0.5745(0.0186) & 0.5488(0.0174) \\
    \textbf{VREx}       & 0.4343(0.0192) & 0.3620(0.0179)  & 0.3858(0.0186) & 0.4453(0.0168) & 0.4737(0.0175) & 0.3810(0.0162)  \\
    \textbf{GroupDRO}   & 0.5695(0.0171) & 0.5782(0.0186) & 0.5488(0.0197) & 0.5800(0.0164)   & 0.5782(0.0172) & 0.5003(0.0189) \\
    \textbf{Coral}      & 0.6167(0.0187) & 0.6032(0.0174) & 0.5288(0.0169) & 0.5735(0.0182) & 0.5767(0.0191) & 0.5292(0.0177) \\
    \textbf{DANN}       & 0.5767(0.0201) & 0.5785(0.0189) & 0.5635(0.0222) & 0.5682(0.0213) & 0.5793(0.0194) & 0.5358(0.0178) \\
    \textbf{Mixup}      & 0.5062(0.0176) & 0.5233(0.0191) & 0.4965(0.0184) & 0.5148(0.0198) & 0.5153(0.0167) & 0.4988(0.0173) \\
    \textbf{DIR}        & 0.6237(0.0218) & 0.5972(0.0197) & 0.5255(0.0176) & 0.6435(0.0188) & 0.6382(0.0193) & 0.4295(0.0169) \\
    \textbf{GSAT}       & 0.4673(0.0183) & 0.4208(0.0215) & 0.4152(0.0197) & 0.3690(0.0172)  & 0.4223(0.0168) & 0.3688(0.0191) \\
    \textbf{CIGA}       & 0.601(0.0178)  & 0.5530(0.0221)  & 0.5663(0.0182) & 0.4982(0.0214) & 0.5368(0.0209) & 0.5197(0.0187) \\
    \textbf{CIGAv1}     & 0.5567(0.0195) & 0.5047(0.0172) & 0.5463(0.0183) & 0.5585(0.0197) & 0.5400(0.0173)   & 0.5435(0.0181) \\
    \textbf{CIGAv2}     & 0.5698(0.0181) & 0.5082(0.0164) & 0.5122(0.0179) & 0.5218(0.0185) & 0.4930(0.0188)  & 0.5180(0.0194)  \\
    \textbf{GALA}       & 0.5976(0.0189) & 0.5894(0.0167) & 0.5861(0.0178) & 0.5833(0.0196) & 0.5796(0.0182) & 0.5535(0.0175) \\
    
    \midrule
    \textbf{PNSIS} &
  \textbf{0.8113(0.0013)} &
  \textbf{0.7777(0.0045)} &
  \textbf{0.7633(0.0121)} &
  \textbf{0.8128(0.0662)} &
  \textbf{0.8000(0.0302)} &
  \textbf{0.7890(0.0025)} \\ 
    \midrule
    \textbf{ORACLE(IID)} & \multicolumn{3}{c|}{88.70(0.1700)} & \multicolumn{3}{c}{88.73(0.2500)}\\
    \bottomrule
    \end{tabular}%
  \label{tab:addlabel}%
}
\end{table*}%

\begin{table*}[ht]
\renewcommand{\arraystretch}{0.50}
\centering
\caption{The ROC-AUC results on seven molecular property classification tasks of the OGB dataset. The values presented are averaged over four replicates with different random seeds. Values in the parenthesis denote the standard errors.}
\label{tab:ogb_cls2}
\resizebox{1.\textwidth}{!}{
\begin{tabular}{@{}c|ccccccc@{}}
\toprule
\textbf{Model}           & Molhiv         & Molbace        & Molbbbp        & Molclintox    & Moltox21       & Molsider       & Moltoxcast     \\ \midrule
\textbf{GCN}             & 0.7580(0.0197)  & 0.7689(0.0323) & 0.6974(0.0153) & 0.9027(0.0134) & 0.7456(0.0035) & 0.5843(0.0034) & 0.6421(0.0069) \\
\textbf{GAT}             & 0.7652(0.0069) & 0.8124(0.0140) & 0.6864(0.0298) & 0.8798(0.0011) & 0.7492(0.0066) & 0.5956(0.0102) & 0.6466(0.0028) \\
\textbf{GraphSAGE}       & 0.7747(0.0115) & 0.7425(0.0248) & 0.6805(0.0126) & 0.8877(0.0066) & 0.7410(0.0035) & 0.6059(0.0016) & 0.6282(0.0067) \\
\textbf{GIN}             & 0.7852(0.0158) & 0.7638(0.0387) & 0.6748(0.0063) & 0.9155(0.0212) & 0.7440(0.0040) & 0.5817(0.0124) & 0.6342(0.0102) \\
\textbf{GIN0}            & 0.7814(0.0121) & 0.7584(0.0239) & 0.6611(0.0094) & 0.9212(0.0255) & 0.7490(0.0015) & 0.5968(0.0148) & 0.6289(0.0019) \\
\textbf{SGC}             & 0.6342(0.0016) & 0.6875(0.0021) & 0.6613(0.0039) & 0.8536(0.0028) & 0.7222(0.0005) & 0.5906(0.0032) & 0.6283(0.0010)  \\
\textbf{JKNet}           & 0.7534(0.0123) & 0.7425(0.0291) & 0.6930(0.0075)  & 0.8558(0.0217) & 0.7418(0.0029) & 0.5818(0.0159) & 0.6357(0.0055) \\
\textbf{DIFFPOOL}        & 0.6408(0.0497) & 0.7525(0.0116) & 0.6935(0.0189) & 0.8241(0.0167) & 0.7325(0.0084) & 0.5758(0.0151) & 0.6217(0.0054) \\
\textbf{CMPNN}           & 0.7711(0.0071) & 0.7215(0.0490)  & 0.6403(0.0172) & 0.7947(0.0461) & 0.7048(0.0107) & 0.5799(0.0080)  & 0.6394(0.0105) \\
\textbf{DIR}             & 0.7672(0.0084) & 0.7834(0.0145) & 0.6467(0.0174) & 0.8129(0.0307) & 0.6966(0.0286) & 0.5794(0.0111) & 0.6196(0.0135) \\
\textbf{StableGNN}   & 0.7779(0.0119) & 07695(0.0327)  & 0.6882(0.0387) & 0.8798(0.0237) & 0.7312(0.0034) & 0.5915(0.0117) & 0.6329(0.0069) \\
\textbf{AttentiveFP}     & 0.7780(0.0195)  & 0.7767(0.0026)  & 0.6555(0.0128) & 0.8335(0.0216) & 0.7934(0.0028) & 
0.6919(0.0148) & 0.7678(0.0037) \\
\textbf{OOD-GNN}     & 0.7950(0.0080)  & 0.8130(0.0120)  & 0.7010(0.0100) & 0.9140(0.0130) & 0.7840(0.0800) & 0.6400(0.0130) & 0.7870(0.0030) \\
\textbf{GIL}             & 0.7908(0.0054) & / & / & / & / & 0.6350(0.0057) & / \\
\textbf{GREA}             & 0.7932(0.0092) & 0.8237(0.0237) & 0.6970(0.0128) & 0.8789(0.0368) & 0.7723(0.0119) & 0.6014(0.0204) & 0.6732(0.0092) \\
\textbf{GALA}        & 0.7788(0.0044)  & 0.7893(0.0037) & 0.6557(0.0335)  & 0.8737(0.0189) & 0.7360(0.0053) & 0.5894(0.0051) & 0.6297(0.0235) \\ \midrule
\textbf{PNSIS} &
  \textbf{0.7953(0.0119)} &
  \textbf{0.8314(0.0181)} &
  \textbf{0.7040(0.0123)} &
  \textbf{0.9323(0.0042)} &
  \textbf{0.7966(0.0023)} &
  \textbf{0.7278(0.0120)} &
  \textbf{0.7900(0.0006)} \\ \bottomrule
\end{tabular}
}
\end{table*}

\subsection{Setup}
\subsubsection{Dataset Description}
We take the graph generalization in the graph classification task into account and consider synthetic and real-world graph datasets with different distribution shifts to evaluate the performance of PNSIS. 

\textbf{SPMotif datasets.} For the simulation dataset, we consider the SPMotif datasets introduced in DIR \cite{wu2022discovering}, where artificial structural shifts and graph size shifts are nested (SPMotif-Struc). To generate different levels of domain shift, we employ the same simulation method in \cite{chen2022learning} the bias based on fully informative invariant feature (FIIF), where the motif and one of the three base graphs (Tree, Ladder, Wheel) are artificially
(spuriously) correlated with a probability of various biases. Furthermore, we also construct a more challenging simulation dataset named SPMotif-Mixed like\cite{chen2022learning}, where the node features are spuriously correlated with a probability of different biases by a fixed number of the corresponding labels.


\textbf{Real-world datasets.}
To evaluate the proposed PNSIS model in real-world scenarios, we further consider the following real-world datasets. First, we consider seven datasets from the OGBG \cite{hu2020open} benchmark, which is a collection of realistic and large-scale molecular data. Additionally, we also consider the graph out-of-distribution (GOOD) benchmark dataset \cite{gui2022good}, which is a systematic benchmark for graph out-of-distribution problems. We select four datasets with different split strategies. 

All graphs in these datasets are pre-processed using
RDKit \cite{landrum2006rdkit}. During preprocessing, these data sets employ a scaffold strategy in the OGBG data set and scaffold/size splitting in the GOOD data set to split molecules based on their two-dimensional structural framework.  This preprocessing procedure will inevitably introduce spurious correlations between functional groups due to the selection
bias of the training set. Please refer to the Appendix \ref{sec:SoD} for the detailed description of the statistics of the dataset.

\subsubsection{Baselines}
We compare the proposed PNSIS method with two three kinds of baselines. Besides the methods devised for out-of-distribution like ERM \cite{vedantam2021empirical}, IRM \cite{arjovsky2019invariant}, VREx \cite{krueger2021out}, GroupDRO \cite{sagawa2019distributionally}, Coral \cite{sun2016deep}, DANN \cite{lempitsky2016domain} and Mixup \cite{zhang2017mixup}, we also consider the conventional methods based on graph neural networks (GNNs) like GCN \cite{kipf2016semi}, GAT \cite{velivckovic2017graph}, GraphSage \cite{hamilton2017inductive}, GIN \cite{xu2018powerful}, SGC \cite{wu2019simplifying}, JKNet \cite{xu2018representation}, DIFFPOOL \cite{ying2018hierarchical}, CMPNN \cite{10.5555/3491440.3491832}, AttentiveFP \cite{xiong2019pushing}, and GSAT \cite{miao2022interpretable}. Then, we further consider the methods that leverage the technique of causal inference like DIR \cite{wu2022discovering}, StableGNN \cite{fan2023generalizing}, CIGA \cite{chen2022learning}, OOD-GNN \cite{li2022ood}. Finally, we consider methods like GALA \cite{chen2023does}, GIL \cite{li2022learning}, and GREA \cite{10.1145/3534678.3539347}, which harness environment augmentation or environment inference to improve generalization. 
We use ADAM optimizer in all experiments and report the mean accuracy as evaluation metrics. All experiments are implemented by Pytorch on a single NVIDIA RTX A5000 24GB GPU.


\begin{table*}[ht]
\centering
\renewcommand{\arraystretch}{0.90}
\caption{The AUC-ROC results on four prediction tasks of the GOOD dataset. The values presented are averaged over four replicates with different random seeds. Values in the parenthesis denote the standard errors.}
\label{tab:ogb_cls3}
\resizebox{1.\textwidth}{!}{
\begin{tabular}{c|cc|cc|c|ccc}
\toprule
        & \multicolumn{2}{c|}{HIV}   
        & \multicolumn{2}{c|}{Motif}
        & \multicolumn{1}{c|}{SST2}
        & \multicolumn{3}{c}{DrugOOD}
        
        \\ \midrule
\textbf{Model}    & scaffold       & size           & basis          & size           & length  & assay &scaffold &size       \\ \midrule
\textbf{ERM}      & 0.6955(0.0239) & 0.5919(0.0229) & 0.6380(0.1036) & 0.5346(0.0408) & 0.8052(0.0113) & 0.7211(0.0123) & 0.6852(0.0145) & 0.6465(0.0167)\\
\textbf{IRM}      & 0.7017(0.0278) & 0.5994(0.0159) & 0.5993(11.46)  & 0.5368(0.0411) & 0.8075(0.0117) & 0.7076(0.0118) & 0.6664(0.0132) & 0.6688(0.0146)  \\
\textbf{VREx}     & 0.6934(0.0354) & 0.5849(0.0228) & 0.6653(0.0404) & 0.5447(0.0342) & 0.8020(0.0139)& 0.6805(0.0138) & 0.6668(0.0150) & 0.6768(0.0163) \\
\textbf{GroupDRO} & 0.6815(0.0284) & 0.5775(0.0286) & 0.6196(0.0827) & 0.5169(0.0222) & 0.8167(0.0045)&0.7177(0.0122) & 0.6897(0.0134) & 0.6635(0.0146) \\
\textbf{Coral}    & 0.7069(0.0225) & 0.5939(0.0290) & 0.6623(0.0901) & 0.5371(0.0275) & 0.7894(0.0122)& 0.7181(0.0121) & 0.6966(0.0133) & 0.6552(0.0145) \\
\textbf{DANN}     & 0.6943(0.0242) & 0.6238(0.0265) & 0.5154(0.0728) & 0.5186(0.0244) & 0.8053(0.0140)& 0.7133(0.0124) & 0.6592(0.0136) & 0.6627(0.0148) \\
\textbf{Mixup}    & 0.7065(0.0186) & 0.5911(0.0311) & 0.6967(0.0586) & 0.5131(0.0256) & 0.8077(0.0103)& \textbf{0.7244(0.0112)} & 0.6970(0.0124) & 0.6648(0.0136) \\
\textbf{DIR}      & 0.6844(0.0251) & 0.5767(0.0375) & 0.3999(0.0550) & 0.4483(0.0400) & 0.8155(0.0106)& 0.6901(0.0143) & 0.6420(0.0156) & 0.6154(0.0169) \\
\textbf{GSAT}     & 0.7007(0.0176) & 0.6073(0.0239) & 0.5513(0.0541) & 0.6076(0.0594) & 0.8149(0.0076) & 0.6978(0.0137) & 0.6450(0.0151) & 0.6092(0.0165)\\
\textbf{CIGAv1}   & 0.6940(0.0239) & 0.6181(0.0168) & 0.6643(0.1131) & 0.4914(0.0834) & 0.8044(0.0124)& 0.7089(0.0126) & 0.6570(0.0138) & 0.6382(0.0152) \\
\textbf{CIGAv2}   & 0.6940(0.0197) & 0.5955(0.0256) & 0.6715(0.0819) & 0.5442(0.0311) & 0.8046(0.0200)&0.6989(0.0135) & 0.6695(0.0147) & 0.6410(0.0160) \\
\textbf{GALA}     & 0.6864(0.0225) & 0.5948(0.0138) & 0.6041(0.015)  & 0.5257(0.0082) & 0.7672(0.0136)& 0.7101(0.0125) & 0.6637(0.0137) & 0.6410(0.0159)\\ \midrule
\textbf{PNSIS} &
  \textbf{0.7167(0.0079)} &
  \textbf{0.6276(0.0042)} &
  \textbf{0.7748(0.0221)} &
  \textbf{0.6326(0.0602)} &
  \textbf{0.8196(0.0020)} &
  \textbf{0.7240(0.0113)}&
  \textbf{0.6991(0.0021)}&
  \textbf{0.6701(0.0171)}
  \\ \bottomrule
\end{tabular}
}
\end{table*}

    

\subsubsection{Experiment Results on Simulation Datasets.} The experiment results on the SPMotif simulation datasets are shown in Table 1. Based on the results of the experiment, we can find that the proposed PNSIS method outperforms the other baselines with a large margin in different biases on the standard SPMotif-Struc dataset and the more challenging SPMotif-Mixed dataset. Specifically, the proposed PNISIS achieves more than $25\%$ averaged improvement on all the simulation datasets, indirectly reflecting that our method can extract the invariant subgraphs with the property of necessity and sufficiency. What is more, we also find that the performance drops with increasing biases, showing that overheavy bias can still influence generalization. Moreover, by comparing the variance of different methods, we can find that the variance of the baselines is large, this is because these methods generate the invariant subgraph by trading off two objects, which might lead to unstable results. In the meanwhile, the variance of our method is much smaller, reflecting the stability of our method.


\subsubsection{Experiment Results on Real-world Datasets.} The experiment results on the OGBG and GOOD datasets are shown in Table 2 and 3. According to the experiment results, we can draw the following conclusions: 1) the proposed PNSIS outperforms all other baselines on all the datasets, which is attributed to both the PNS restriction for PNS invariant subgraphs and the ensemble training strategy with the help of the spurious subgraphs. 2) Some GNN-based methods such as GCN and GIN do not achieve the ideal performance, reflecting that these methods have limited generalization. 3) The causality-based baselines also achieve comparable performance and the methods based on environmental data-augmentation achieve the closest results, reflecting the usefulness of the environment augmentation. However, since these methods can hardly extract the necessary and sufficient invariant subgraphs, so some experiment results of these methods like Molclintox, Molsider, and Moltoxcast can hardly achieve ideal performance. 
\begin{figure}[t]
    \centering
\includegraphics[width=0.9\columnwidth]{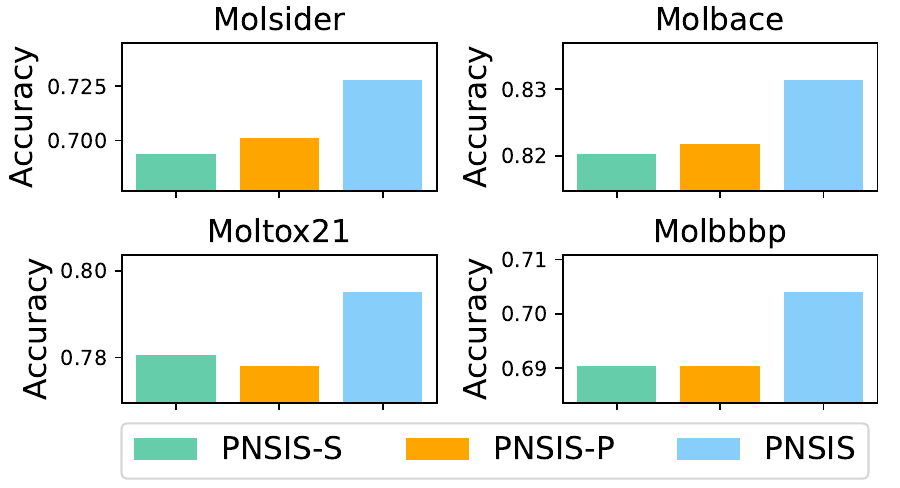}
\vspace{-2mm}
    \caption{Ablation study on the Molsider, Molbace, Moltox21, and Molbbbp datasets in the OGBG benchmark. We explore the impact of different components in the PNSIS method.\textbf{}}
\vspace{-5mm}
    \label{fig:ablation}
\end{figure}

\subsection{Ablation Study}
To answer Q2 and Q3 to show if the proposed PNS invariant extractor with PNS upper bound restriction and the ensemble strategy with spurious subgraphs benefit the generalization performance, we also devise the two model variants. 1) \textbf{PNSIS-S}: we remove the PNS upper bound restriction from the standard PNSIS. 2) \textbf{PNSIS-P} we remove the ensemble strategy from the standard PNSIS. 
Experiment results on four datasets in the OBGB benchmark are shown in Figure \ref{fig:ablation}. According to the experiment results, we can find that both the PNS-upper bound restriction and the ensemble strategy play an important role in generalization performance, illustrating the effectiveness of the components in the proposed frameworks.

\subsection{Visualization}
We further provide visualization of the invariant latent subgraphs extracted by the proposed PNSIS, which is shown in Figure \ref{fig:vis}. According to the experimental results, we can find that our method can capture the unique atoms of the molecules, i.e. the nitrogen atom and the phosphorus atom, which play a significant role in the molecular property, making it possible to provide inspiration and explanations in the field of chemistry.
\begin{figure}
    \centering
\includegraphics[width=0.95\columnwidth]{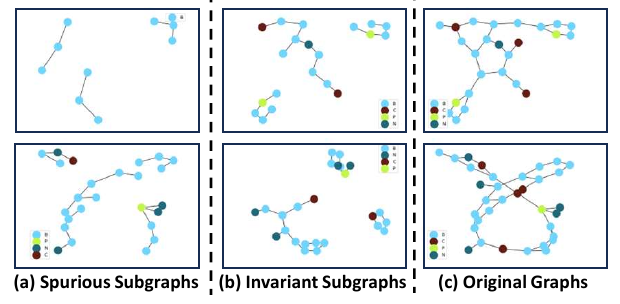}
\vspace{-3mm}
    \caption{Visualization of molecule examples in the OGBG benchmark. Modes with different colors denote different atoms, and edges denote different chemical bonds.}
    \vspace{-3mm}
    \label{fig:vis}
\end{figure}

\section{Conclusion}
This paper presents a unified framework for graph out-of-distribution, which leverages probability of necessity and sufficiency for invariant subgraph learning and involve spuriousity subgraph for ensemble inference. Based on the conventional data generation process, we prove that the necessity and sufficiency invariant subgraphs can be learned by optimizing the proposed upper bound. By unifying the invariant subgraph extractor and the ensemble inference in the test phase, the proposed PNSIS framework shows outstanding experiment results on the simulation and real-world datasets, highlighting its effectiveness. In summary, this paper takes a meaningful step for the combination of causality and graph representation learning.


\clearpage

\bibliographystyle{ACM-Reference-Format}
\bibliography{citation}

\clearpage
\appendix

\section{Background of Power-sum Multi-symmetric Polynomials}\label{sec:pmp}
Given a graph $G_n=(\mathcal{V}_n, \mathcal{E}_n, \mathbf{A}_n)$, use the following permutation-invariant function to get its graph representation $\mathbf{h}_n$. This function adopts \textbf{P}ower-sum \textbf{M}ulti-symmetric \textbf{P}olynomials (PMP) to encode multiset, which is as powerful as the $2$-FWL (equivalent to $3$-WL) graph isomorphism test, as follows.
\begin{equation}\label{equ:fwl2}
    \mathbf{h}_n  =  \sum_{(i_1, i_2) \in [V_n]^2}\mathbf{C}_{n_{i_1, i_2, :}} \big\| \text{PMP}(\mmset{ (\mathbf{C}_{n_{k, i_2, :}}, \mathbf{C}_{n_{i_1, k, :}}) \big\vert k \in [V_n]}),
\end{equation}
where $\mathbf{C}_{n_{i_1, i_2, :}} \in \mathbb R^D$ is the feature (or color) of $k$-tuple $\boldsymbol{i}=(i_1, i_2)$.
\begin{equation}
     \text{PMP}(\mmset{ (\mathbf{C}_{n_{k, i_2, :}}, \mathbf{C}_{n_{i_1, k, :}}) \big\vert k \in [V_n]}) = \big\|_{j=1}^{D'} \mathbf{z}^j_{(i_1, i_2)},
\end{equation}
\begin{equation}
    \mathbf{z}^j_{(i_1, i_2)} = \sum_{k \in [V_n]} \tau_1(\mathbf{C}_n)_{k, i_2, j} \odot  \tau_2(\mathbf{C}_n)_{i_1, k, j},
\end{equation}
where $D' = \left(\begin{smallmatrix}N+2D\\ 2D\end{smallmatrix}\right)$, $\tau_1, \tau_2: \mathbb R^{D} \to \mathbb R^{D\cdot D'}$ are two polynomial maps, 
$\tau_1(\mathbf{C}_n) = \big\|_{j=1}^{D'} \mathbf{C}_n^{\mathbf{a}_{j}}$, 
$\tau_2(\mathbf{C}_n) = \big\|_{j=1}^{D'} \mathbf{C}_n^{\mathbf{b}_{j}}$, $\mathbf{a}_j, \mathbf{b}_j \in [N]^{D}$, $|\mathbf{a}_j| + |\mathbf{b}_j| \le N$, $|\mathbf{a}_j| = \sum^{D}_{p=1} a_{jp}$, $[N]=\{1, 2, ..., N\}$.

\section{Formula Derivation for Ensemble Prediction}\label{sec:ensemble}
\begin{equation}
   \begin{aligned}
p(Y \mid G^c, G^s) & \propto p(G^c, G^s \mid Y) p(Y) \\
& =p(G^c \mid Y) p(G^s \mid Y) p(Y) \\
& \propto \frac{p(Y \mid G^c) p(Y \mid G^s)}{p(Y)}
\end{aligned} 
\end{equation}

The premise for the second line to hold is that given $Y$, then $G^c$, and $G^s$ are independent of each other.

\section{Proof of Theorem \ref{thm:distance}}\label{sec:distance}
\begin{proof}
In Definition \ref{def:distance}, our paper presents the Graph Structure Distance between the source and target domains as shown in  Eq.~\ref{equ:distance1}. Furthermore, we adopt the total variation distance as the distance metric for the distribution metric. For the graph structure distributions of environments $E_{\mathcal{T}_i}$ and $E_{\mathcal{T}_j}$, the Graph Structure Distance can be expressed as follows:
\begin{equation}
    \label{equ:TV}
    \begin{aligned}
    &dist(p_{\mathbf h|E_{\mathcal{T}_i}}, p_{\mathbf h|E_{\mathcal{T}_j}})\\
    &=TotalVariation(p_{\mathbf h|E_{\mathcal{T}_i}}, p_{\mathbf h|E_{\mathcal{T}_j}})\\
    &=\frac{1}{2} \sum |p_{\mathbf h|E_{\mathcal{T}_i}} - p_{\mathbf h|E_{\mathcal{T}_j}}|
    \end{aligned}
\end{equation}

(1) Graph Structure Distance satisfies the first axiom: the distance metric function satisfies non-negativity.\\
From Eq.\ref{equ:TV}, it is easy to see that $|p_{\mathbf h|E_{\mathcal{T}_i}} - p_{\mathbf h|E_{\mathcal{T}_j}}|$ is always greater than or equal to 0.
And since $\|\mathbf{X}^i - \mathbf{X}^j\|_2$ is always greater than 0, it is proved that the graph structure distance satisfies non-negativity.

(2) Graph Structure Distance satisfies the second axiom: The distance metric function satisfies the commutative property.\\
From Eq.~\ref{equ:TV} it can be deduced that $dist(p_{\mathbf h|E_{\mathcal{T}_i}}, p_{\mathbf h|E_{\mathcal{T}_j}})$ can be expressed as follows:
\begin{equation}\label{equ:secondaxiom}
    \begin{aligned}
    &dist(p_{\mathbf h|E_{\mathcal{T}_i}}, p_{\mathbf h|E_{\mathcal{T}_j}}) \\
    &=TotalVariation(p_{\mathbf h|E_{\mathcal{T}_i}}, p_{\mathbf h|E_{\mathcal{T}_j}}) \\&=
    \frac{1}{2} \sum |p_{\mathbf h|E_{\mathcal{T}_i}} - p_{\mathbf h|E_{\mathcal{T}_j}}|\\
    &=\frac{1}{2} \sum |p_{\mathbf h|E_{\mathcal{T}_j}} - p_{\mathbf h|E_{\mathcal{T}_i}}|\\
    &=TotalVariation(p_{\mathbf h|E_{\mathcal{T}_j}}, p_{\mathbf h|E_{\mathcal{T}_i}})\\
    &=dist(p_{\mathbf h|E_{\mathcal{T}_j}}, p_{\mathbf h|E_{\mathcal{T}_i}})
    \end{aligned}
\end{equation}
For $\mathbb{E}_{\mathbf{X}_i \sim p_{\mathbf X|E_{\mathcal{T}_i}}}\big[\mathbb E_{\mathbf{X}_j \sim p_{\mathbf X|E_{\mathcal{T}_j}}} [\|\mathbf{X}_i - \mathbf{X}_j\|_2]\big]$, we extend it as follows.
\begin{equation}
    \label{equ:E1}
    \begin{aligned}
    &\mathbb E_{\mathbf{X}_i \sim p_{\mathbf X|E_{\mathcal{T}_i}}}\big[\mathbb E_{\mathbf{X}_j \sim p_{\mathbf X|E_{\mathcal{T}_j}}} [\|\mathbf{X}_i - \mathbf{X}_j\|_2]\big]\\
    &=\int \left( \int \|\mathbf{X
    }_i - \mathbf{X}_j\|_2 \, dp_{\mathbf X|E_{\mathcal{T}_j}}(\mathbf{X}_j) \right) \, dp_{\mathbf X|E_{\mathcal{T}_i}}(\mathbf{X}_i)
    \end{aligned}
\end{equation}
Next, we use Fubini's theorem, which allows us to swap the order of integrals, provided the product function is non-negative. Here $\|\mathbf{X}_i - \mathbf{X}_j\|_2$ is non-negative.
Therefore, Eq.~\ref{equ:E} is supplemented as follows.
\begin{equation}
    \label{equ:E}
    \begin{aligned}
    &\mathbb E_{\mathbf{X}_i \sim p_{\mathbf X|E_{\mathcal{T}_i}}}\big[\mathbb E_{\mathbf{X}_j \sim p_{\mathbf X|E_{\mathcal{T}_j}}} [\|\mathbf{X}_i - \mathbf{X}_j\|_2]\big]\\
    &=\int \left( \int \|\mathbf{X
    }_i - \mathbf{X}_j\|_2 \, dp_{\mathbf X|E_{\mathcal{T}_j}}(\mathbf{X}_j) \right) \, dp_{\mathbf X|E_{\mathcal{T}_i}}(\mathbf{X}_i)\\
    &=\int \left( \int \|\mathbf{X
    }_i - \mathbf{X}_j\|_2 \, dp_{\mathbf X|E_{\mathcal{T}_i}}(\mathbf{X}_i)  \right)\, dp_{\mathbf X|E_{\mathcal{T}_j}}(\mathbf{X}_j)\\
    &=\mathbb E_{\mathbf{X}_j \sim p_{\mathbf X|E_{\mathcal{T}_j}}}\big[\mathbb E_{\mathbf{X}_i \sim p_{\mathbf X|E_{\mathcal{T}_i}}} [\|\mathbf{X}_i - \mathbf{X}_j\|_2]\big]\\
    \end{aligned}
\end{equation}
Therefore, it is proved that the graph structure distance satisfies the commutative property.

(3) Graph Structure Distance (GSD) satisfies the third axiom: The distance metric function satisfies the triangle inequality.

From Eq.~\ref{equ:TV} it can be deduced that $dist(p_{\mathbf h|E_{\mathcal{T}_i}}, p_{\mathbf h|E_{\mathcal{T}_j}})$ can be expressed as follows, where $p_{\mathbf h|E_{\mathcal{T}_k}}$ is an assumed distribution.
\begin{equation}
    \label{equ:secondaxiom1}
    \begin{aligned}
    &dist(p_{\mathbf h|E_{\mathcal{T}_i}}, p_{\mathbf h|E_{\mathcal{T}_j}})=TotalVariation(p_{\mathbf h|E_{\mathcal{T}_i}}, p_{\mathbf h|E_{\mathcal{T}_j}}) \\
    =&\frac{1}{2} \sum |p_{\mathbf h|E_{\mathcal{T}_i}} - p_{\mathbf h|E_{\mathcal{T}_j}}| \\
    =&\frac{1}{2} \sum |p_{\mathbf h|E_{\mathcal{T}_i}} - p_{\mathbf h|E_{\mathcal{T}_k}} + p_{\mathbf h|E_{\mathcal{T}_k}} - p_{\mathbf h|E_{\mathcal{T}_j}}|\\
    \leq&
    \frac{1}{2} \sum |p_{\mathbf h|E_{\mathcal{T}_i}} - p_{\mathbf h|E_{\mathcal{T}_k}}| + \frac{1}{2} \sum |p_{\mathbf h|E_{\mathcal{T}_k}} - p_{\mathbf h|E_{\mathcal{T}_j}}|\\
    =&TotalVariation(p_{\mathbf h|E_{\mathcal{T}_i}}, p_{\mathbf h|E_{\mathcal{T}_k}})+TotalVariation(p_{\mathbf h|E_{\mathcal{T}_k}}, p_{\mathbf h|E_{\mathcal{T}_j}})\\
    =&dist(p_{\mathbf h|E_{\mathcal{T}_i}}, p_{\mathbf h|E_{\mathcal{T}_k}})+dict(p_{\mathbf h|E_{\mathcal{T}_k}}, p_{\mathbf h|E_{\mathcal{T}_j}})
    \end{aligned}
\end{equation}
In addition, we introduce an auxiliary variable $\mathbf{X}_k$ with the distribution $p_{\mathbf h|E_{\mathcal{T}_k}}$ representing the graph structure representation in domain $E_{\mathcal{T}_k}$, and it satisfies $p_{\mathbf h|E_{\mathcal{T}_k}}(\mathbf{X}_k, \mathbf{X}_i) = p_{\mathbf h|E_{\mathcal{T}_k}}(\mathbf{X}_k) \cdot p_{\mathbf h|E_{\mathcal{T}_k}}(\mathbf{X}_i)$. That is, we introduce a distribution $p_{\mathbf h|E_{\mathcal{T}_k}}$independent of $p_{\mathbf h|E_{\mathcal{T}_j}}$ but correlated with $p_{\mathbf h|E_{\mathcal{T}_i}}$ to construct the auxiliary variable $\mathbf{X}_k$. Now, we expand on the expectation part of Eq.~\ref{equ:distance1}.
\begin{equation}
    \label{equ:third}
    \begin{aligned}
 &\mathbb E_{\mathbf{X}_j \sim p_{\mathbf X|E_{\mathcal{T}_j}}}\left[\|\mathbf{X}_j - \mathbf{X}_k\|_2\right] \\
&= \mathbb E_{\mathbf{X}_j, \mathbf{X}_k \sim p_{\mathbf X|E_{\mathcal{T}_j}}p_{\mathbf X|E_{\mathcal{T}_k}}}\left[\|\mathbf{X}^j - \mathbf{X}^k\|_2\right] 
+
\mathbb E_{\mathbf{X}_k, \mathbf{X}_i \sim p_{\mathbf X|E_{\mathcal{T}_k}}p_{\mathbf X|E_{\mathcal{T}_i}}}\left[\|\mathbf{X}_k - \mathbf{X}_i\|_2\right]
    \end{aligned}
\end{equation}
Next, we apply the basic triangle inequality to Eq.~\ref{equ:third}.
\begin{equation}
    \label{equ:third1}
    \begin{aligned}
 &\mathbb E_{\mathbf{X}_j \sim p_{\mathbf X|E_{\mathcal{T}_j}}}\left[\|\mathbf{X}_j - \mathbf{X}_k\|_2\right]\\
&\leq 
\mathbb E_{\mathbf{X}_j \sim p_{\mathbf X|E_{\mathcal{T}_j}}}\left[\|\mathbf{X}_j - \mathbf{X}_k\|_2\right]
+ \mathbb E_{\mathbf{X}_k \sim p_{\mathbf X|E_{\mathcal{T}_k}}}\left[\|\mathbf{X}_k - \mathbf{X}_i\|_2\right] \\
    \end{aligned}
\end{equation}
Finally, we can conclude as follows.
\begin{equation}
    \label{equ:third2}
    \begin{aligned}
&\mathbb E_{\mathbf{X}_i \sim p_{\mathbf X|E_{\mathcal{T}_i}}}
\left[E_{\mathbf{X}_j \sim p_{\mathbf X|E_{\mathcal{T}_j}}}\left[\|\mathbf{X}_j - \mathbf{X}_i\|_2\right]\right] 
&\leq \mathbb E_{\mathbf{X}_j \sim p_{\mathbf X|E_{\mathcal{T}_j}}}\left[\|\mathbf{X}_j - \mathbf{X}_k\|_2\right] \\
&+ \mathbb E_{\mathbf{X}_k \sim p_{\mathbf X|E_{\mathcal{T}_k}}}\left[\|\mathbf{X}_j - \mathbf{X}_k\|_2\right] \\
&+ \mathbb E_{\mathbf{X}_k \sim p_{\mathbf X|E_{\mathcal{T}_k}}}\left[\|\mathbf{X}_k - \mathbf{X}_i\|_2\right] \\
    \end{aligned}
\end{equation}
Therefore, it is proved that the graph structure distance satisfies the triangle inequality.
\end{proof}

\section{Proof of Theorem \ref{thm:bound}}\label{sec:bound}
\begin{theorem}\label{thm:bound1}
\textbf{(Generalization Bound)}
    We make the following assumption:
    \begin{itemize}
        \item A2: For two distinct environment distributions $p_{\mathbf{A},\mathbf{X},Y, |E_{\mathcal{T}_i}}$ and $p_{\mathbf{A},\mathbf{X},Y, |E_{\mathcal{T}_i}}$,        
        assume a positive value $K$ exists that satisfies the following inequality:
\begin{equation}
\begin{aligned}
&|p_{\mathbf{A},\mathbf{X},Y, |E_{\mathcal{T}_i}}-p_{\mathbf{A},\mathbf{X},Y, |E_{\mathcal{T}_j}}|\\
 \leq &K\cdot\big(|
 p_{\mathbf{h} |E_{\mathcal{T}_i}}
 -
p_{\mathbf{h} |E_{\mathcal{T}_j}}
 |+\mathbb{E}_{\mathbf{X}_i \sim p_{\mathbf{X}|E_{\mathcal{T}_i}}}\big[\mathbb E_{\mathbf{X}_j \sim p_{\mathbf{X} |E_{\mathcal{T}_j}}} [\|\mathbf{X}_i - \mathbf{X}_j\|_2]\big]\big) \\
= &K \cdot d_{sd}^{\mathcal{T}_i \leftrightarrow \mathcal{T}_j}\left(\mathcal{T}_i, \mathcal{T}_j\right).
\end{aligned}
\end{equation}
    \end{itemize}
Based on the aforementioned definition and assumption, we
propose the generalization bound for PNS risk in Eq.~\ref{equ:estimated_pns}.
\begin{equation}\label{equ:bound22}
\begin{aligned}
    &r^{\mathcal{T}_i}_{ns}(\Theta^{sf}, \Phi^{sf}, \Theta^{nc}, \Phi^{nc})
    \leq  r^{\mathcal{T}_j}_{ns}(\Theta^{sf}, \Phi^{sf}, \Theta^{nc}, \Phi^{nc}) \\&+ K\cdot \big(|p_{\mathbf{h}|E_{\mathcal{T}_i}}-p_{\mathbf{h}|E_{\mathcal{T}_j}}|
    +\mathbb{E}_{\mathbf{X}_i \sim p_{\mathbf{X}|E_{\mathcal{T}_i}}}\big[\mathbb E_{\mathbf{X}_j \sim p_{\mathbf{X}|E_{\mathcal{T}_j}}} [\|\mathbf{X}_i - \mathbf{X}_j\|_2]\big]\big) \\
    = &r^{\mathcal{T}_j}_{ns}(\Theta^{sf}, \Phi^{sf}, \Theta^{nc}, \Phi^{nc}) + K \cdot d_{sd}^{\mathcal{T}_i \leftrightarrow \mathcal{T}_j}\left(\mathcal{T}_i, \mathcal{T}_j\right) + \lambda,   
\end{aligned}
\end{equation}
where $K$, $\lambda$ are constants.
\end{theorem}
\begin{proof}
Let $\eta: \mathcal{A} \times \mathcal{X}  \to \{0, 1\}$ be the label function. To simplify the description, we overload the symbol $r^{\mathcal{T}_i}_{ns}$ in Eq.~\ref{equ:estimated_pns} as a function of input $G_n$ and label function $\eta$, as follows.
\begin{equation}\label{equ:estimated_pns2}
\begin{aligned}
    &r^{\mathcal{T}_i}_{ns}(G_n, \eta) :=\mathbb E_{(\mathbf{A}_n, \mathbf{X}_n, y_n) \sim p_{\mathbf{A}, \mathbf{X}, Y|E_{\mathcal{T}_i}}}\big[\\
    &\mathbb{E}_{G_j^c \sim P(G^c |\mathbf{A}_n, \mathbf{X}_n; \Theta^{sf})}\mathbb{I}[g(G_j^c; \Phi^{sf}) \ne  \eta(\mathbf{A}_n, \mathbf{X}_n)]\\
    &+ \mathbb{E}_{G_j^c \sim P(G^c |\mathbf{A}_n, \mathbf{X}_n; \Theta^{nc})}\mathbb{I}[g(G_j^c; \Phi^{nc}) =  \eta(\mathbf{A}_n, \mathbf{X}_n)]\big]
\end{aligned}
\end{equation}

We use the shorthand $r^{\mathcal{T}_i}(G_n) = r^{\mathcal{T}_i}_{ns}(G_n, \eta)$ if the label function defined on the same environment as $r^{\mathcal{T}_i}$.
\begin{equation}
\begin{small}
    \begin{aligned}
    &r^{\mathcal{T}_i}(G_n)=r^{\mathcal{T}_i}(G_n)+r^{\mathcal{T}_j}(G_n)-r^{\mathcal{T}_j}(G_n)+r^{\mathcal{T}_j}(G_n, \eta^{\mathcal{T}_i})-r^{\mathcal{T}_j}(G_n, \eta^{\mathcal{T}_i})\\
    \leq &r^{\mathcal{T}_j}(G_n)+r^{\mathcal{T}_j}(G_n,\eta^{\mathcal{T}_i})
    -r^{\mathcal{T}_j}(G_n,\eta^{\mathcal{T}_j})
    +\left|r^{\mathcal{T}_i}(G_n) - r^{\mathcal{T}_j}(G_n,\eta^{\mathcal{T}_i})\right|\\
    \leq &r^{\mathcal{T}_j}(G_n)+r^{\mathcal{T}_j}(G_n,\eta^{\mathcal{T}_i})
    +\left|r^{\mathcal{T}_i}(G_n) - r^{\mathcal{T}_j}(G_n,\eta^{\mathcal{T}_i})\right|\\
    \leq &r^{\mathcal{T}_j}(G_n)+r^{\mathcal{T}_j}(G_n,\eta^{\mathcal{T}_i})
    +\int |p_{\mathbf{A},\mathbf{X}|E_{\mathcal{T}_i}} - p_{\mathbf{A},\mathbf{X}|E_{\mathcal{T}_i}}||g_{nc}(f_{nc}(G))\\
    &-\eta^{\mathcal{T}_i}(G)|\cdot (1-|g_{sf}(f_{sf}(G))-\eta^{\mathcal{T}_i}(G)|)d G\\
    \leq &r^{\mathcal{T}_j}(G_n)+r^{\mathcal{T}_j}(G_n,\eta^{\mathcal{T}_i})
    +\int |p_{\mathbf{A},\mathbf{X}|E_{\mathcal{T}_i}} - p_{\mathbf{A},\mathbf{X}}|d G\\
    \leq &r^{\mathcal{T}_j}(G_n)+ K\cdot\big(|
 p_{\mathbf{h} |E_{\mathcal{T}_i}}
 -
p_{\mathbf{h} |E_{\mathcal{T}_j}}
 |+\mathbb{E}_{\mathbf{X}_i \sim p_{\mathbf{X}|E_{\mathcal{T}_i}}}\big[\mathbb E_{\mathbf{X}_j \sim p_{\mathbf{X} |E_{\mathcal{T}_j}}} [\|\mathbf{X}_i - \mathbf{X}_j\|_2]\big]\big)\\
 &+ \mathbb E_{G \sim p_{G|E_{\mathcal{T}_i}}}[\eta^{\mathcal{T}_i}(G) - \eta^{\mathcal{T}_j}(G)]\\
= &r^{\mathcal{T}_j}(G_n)+K \cdot d^{\mathcal{T}_i \leftrightarrow \mathcal{T}_j}_{SD}\left(\mathcal{T}_i, \mathcal{T}_j\right) + \lambda,
\end{aligned}
\end{small}
\end{equation}
where $\lambda = \epsilon_{S_i}\left(\eta_{S_i}, \eta_{S_j}\right)$ is a constant.
\end{proof}

\section{The Statistics of Datasets}\label{sec:SoD}
The statistics of the real-world datasets are shown in Table 4.
\begin{table*}[t]
\centering
\caption{Information about the datasets used in experiments. The number of nodes and edges are
taking average among all graphs.}
\label{tab:Statistics}
\resizebox{1.\textwidth}{!}{
\begin{tabular}{c|ccccccc}
\toprule
Category & Name & \#Graphs & Average \#Nodes & Average \#Edges & Task Type & Split Method & Metric \\
\midrule
\multirow{7}{*}{OGBG} & HIV & 41127 & 25.5 & 54.9 & Binary Classification & scaffold & ROC-AUC \\
 & BACE & 1513 & 34.1 & 73.7 & Binary Classification & scaffold & ROC-AUC \\
 & BBBP & 2039 & 24.1 & 51.9 & Binary Classification & scaffold & ROC-AUC \\
 & ClinTox & 1477 & 26.2 & 55.8 & Binary Classification & scaffold & ROC-AUC \\
 & Tox21 & 7831 & 18.6 & 38.6 & Binary Classification & scaffold & ROC-AUC \\
 & SIDER & 1427 & 33.6 & 70.7 & Binary Classification & scaffold & ROC-AUC \\
 & toxcast & 8576 & 18.8 & 38.5 & Binary Classification & scaffold & ROC-AUC \\
\midrule
\multirow{6}{*}{GOOD} & \multirow{2}{*}{hiv} & 32903 & 25.3 & 54.4 & Binary Classification & scaffold & ROC-AUC \\
 & & 32903 & 24.9 & 53.6 & Binary Classification & size & ROC-AUC \\
 & SST2 & 44778 & 10.20 & 18.40 & Binary Classification & Length &  ACC\\
 & EC50-Assay & 10464 & 40.89 & 87.18 & Binary Classification & Assay & ROC-AUC \\
& EC50-Scoffolf & 8228 & 35.54 & 75.56 & Binary Classification & Scoffolf & ROC-AUC \\
& EC50-Size & 10189 & 35.12 & 75.30 & Binary Classification & Size & ROC-AUC \\
\bottomrule
\end{tabular}
}
\end{table*}


\end{document}